\documentclass{article}

\usepackage{microtype}
\usepackage{graphicx}
\usepackage{subfigure}
\usepackage{booktabs} 

\usepackage{hyperref}

\usepackage[accepted]{icml2023}

\usepackage{amsmath}
\usepackage{amssymb}
\usepackage{mathtools}
\usepackage{amsthm}

\usepackage[capitalize,noabbrev]{cleveref}

\theoremstyle{plain}
\newtheorem{theorem}{Theorem}[section]

\newtheorem{lemma}[theorem]{Lemma}

\theoremstyle{definition}

\theoremstyle{remark}
\newtheorem{remark}[theorem]{Remark}

\usepackage[disable,textsize=tiny]{todonotes}

\icmltitlerunning{Statistical Foundations of Prior-Data Fitted Networks}

\usepackage{bm}
\usepackage{dsfont}

\newcommand{\ind}{\mathds{1}}
\newcommand{\R}{\mathds{R}}
\newcommand{\N}{\mathds{N}}

\newcommand{\E}{\mathds{E}}

\providecommand{\G}{}

\newcommand{\bnull}{\bm{0}}
\newcommand{\ba}{\bm{a}}
\newcommand{\bb}{\bm{b}}

\newcommand{\bs}{\bm{s}}

\newcommand{\bu}{\bm{u}}
\newcommand{\bv}{\bm{v}}

\newcommand{\bx}{\bm{x}}

\newcommand{\bz}{\bm{z}}

\newcommand{\bV}{\bm{V}}

\newcommand{\bX}{\bm{X}}

\newcommand{\Dcal}{\mathcal{D}}

\newcommand{\Ncal}{\mathcal{N}}

\newcommand{\Pcal}{\mathcal{P}}
\newcommand{\Qcal}{\mathcal{Q}}

\newcommand{\Xcal}{\mathcal{X}}
\newcommand{\Ycal}{\mathcal{Y}}

\newcommand{\bgamma}{\bm{\gamma}}

\newcommand{\btheta}{\bm{\theta}}

\renewcommand{\bar}{\overline}
\providecommand{\Pr}{}
\renewcommand{\Pr}{\mathbb{P}}
\newcommand{\var}{{\mathds{V}\mathrm{ar}}}

\newcommand{\cov}{{\mathds{C}\mathrm{ov}}}
\newcommand{\corr}{{\mathrm{Corr}}}

\newcommand{\wh}[1]{\widehat{#1}}
\newcommand{\wt}[1]{\widetilde{#1}}

\newcommand{\KL}{\operatorname{KL}}

\newcommand{\softmax}{\operatorname{SoftMax}}
\newcommand{\layernorm}{\operatorname{LayerNorm}}
\newcommand{\avg}{\operatorname{avg}}
\newcommand{\relu}{\operatorname{ReLu}}

\newcommand{\fig}[1][1]{
  \includegraphics[width = #1\textwidth]
}

\makeatletter
\def\blfootnote{\gdef\@thefnmark{}\@footnotetext}
\makeatother

\usepackage{amsmath}
\usepackage{amssymb}
\usepackage{amsthm}
\usepackage{mathtools}
\usepackage{enumitem}

\newcommand{\test}{{\text{test}}}

\begin{document}

\twocolumn[
\icmltitle{Statistical Foundations of Prior-Data Fitted Networks}



\icmlsetsymbol{equal}{*}

\begin{icmlauthorlist}
\icmlauthor{Thomas Nagler}{lmu,mcml}
\end{icmlauthorlist}

\icmlaffiliation{lmu}{Department of Statistics, LMU Munich, Munich, Germany}
\icmlaffiliation{mcml}{Munich Center for Machine Learning, Munich, Germany}

\icmlcorrespondingauthor{Thomas Nagler}{t.nagler@lmu.de}

\icmlkeywords{Machine Learning, Statistical Learning, Generalization, Bayesian, Transformer, In-context Learning}

\vskip 0.3in
]



\printAffiliationsAndNotice{}  

\begin{abstract}
	Prior-data fitted networks (PFNs) were recently proposed as a new paradigm for machine learning. Instead of training the network to an observed training set, a fixed model is pre-trained offline on small, simulated training sets from a variety of tasks. The pre-trained model is then used to infer class probabilities in-context on fresh training sets with arbitrary size and distribution.
	Empirically, PFNs achieve state-of-the-art performance on tasks with similar size to the ones used in pre-training. Surprisingly, their accuracy further improves when passed larger data sets during inference.
	This article establishes a theoretical foundation for PFNs and illuminates the statistical mechanisms governing their behavior. While PFNs are motivated by Bayesian ideas, a purely frequentistic interpretation of PFNs as pre-tuned, but untrained predictors explains their behavior. 
	A predictor's variance vanishes if its sensitivity to individual training samples does and the bias vanishes only if it is appropriately localized around the test feature.
	The transformer architecture used in current PFN implementations ensures only the former. These findings shall prove useful for designing architectures with favorable empirical behavior.
\end{abstract}

\section{Introduction}

\subsection{PFNs in a Nutshell}

Prior-data fitted networks (PFNs) were proposed by \citet{muller2021transformers} as a new approach to machine learning, motivated by ideas from Bayesian nonparametrics and meta-learning.
The goal is to compute a posterior predictive distribution (PPD) for a test feature given observed training data.
To approximate the PPD, a transformer network is trained offline through meta-learning.
After simulating several training data sets from a variety of tasks, a transformer network imitating the PPD on these sets is trained.
After this pre-training phase, the network is treated as fixed. In the inference phase, a fresh training set and some test features are passed to the pre-trained network, which computes predictions for the test labels in a single forward pass.

This approach is different from usual machine learning methods. Here, one would
set up a model for the relationship between label and feature, and
train the model parameters on a specific data set. The main benefit of PFNs is that no training or tuning is necessary at the inference stage, and predictions are delivered in split seconds.

\subsection{Empirical Findings}

Empirically, \citet{muller2021transformers} found that PFNs can indeed approximate a given PPD and perform well on real prediction tasks. While this pilot study was limited to tiny data sets, the TabPFN model of \citet{hollmann2022tabpfn} made a leap forward to classification tasks on moderately large tabular data sets. In particular, they pre-train a network with simulated data sets of size up to $n \approx 1000$ and report state-of-the-art performance on several benchmarks.
And surprisingly, the network's predictions continue to improve at the inference stage when fed data sets with more than 1000 samples.
This is an example of \emph{in-context learning (ICL)}: a pre-trained network learns from the context provided in the prompt (here: the fresh training data) without updating its parameters.

\subsection{Summary}

The main contribution of this work is establishing the theoretical foundation for PFNs and identifying statistical mechanisms explaining their empirical behavior.

\begin{itemize}
	\item \textbf{Theoretical framework} (\cref{sec:framework}): As a preliminary step, we give precise definitions of the PPD, the statistical model behind it, and its PFN approximation.
	\item \textbf{When PPDs can learn} (\cref{sec:ppd}):
	Since PPDs are the main motivation behind PFNs, we first ask when a PPD can learn from  training data. This can be approached from the perspective of Bayesian nonparametrics \citep{ghosal2017fundamentals}. If the prior has large enough support and does not concentrate too much away from the true hypothesis, one can guarantee that the PPD converges to a close approximation of the true predictive distribution.

	\item \textbf{How PFNs approximate PPDs} (\cref{sec:approx}):
	The optimal PFN approximation is characterized by a Kullback-Leibler criterion.
	To allow for accurate approximation of the conditional class probabilities, we need sufficiently complex PFN models and prior.
	Practically, a PFN is trained on simulated data sets.
	The larger these data sets are, the more complex the PPD we approximate. The training set size can therefore be understood as a regularizer on the expected complexity of the network. 
	The Monte-Carlo approximation of the optimal PFN is discussed briefly, but rather uneventful and of minor importance for PFNs' inference behavior. 
	
	\item \textbf{Why PFNs can learn} (\cref{sec:learning}): The most intriguing question is: Why can a pre-trained network still learn in the inference phase? Although we know now why a PPD does, a PFN is not a valid PPD, and it is only trained to approximate the PPD for limited training sizes. The learning phenomenon can be understood through a purely frequentistic interpretation of PFNs as untrained predictors with many hyperparameters. During pre-training, these hyperparameters are tuned to be optimal for a set of tasks defined by the `prior'.
	Whether the PFN predictor can learn at inference depends on its structural properties.
	We show that the variance of a fixed network vanishes if its sensitivity to individual samples does, and that the network's bias can only vanish if it is sufficiently localized.
	
	\item \textbf{Insights on specific PFNs} (\cref{sec:examples}): We look at some specific PFN models: window smoothers, classification trees, and transformer networks.
	The examples cover cases where the bias is constant, increasing, or decreasing with $n$.
	We show that if the model is well-designed, it can implicitly select or average over sub-models, making the bias decrease with the sample size. Transformer networks allow for vanishing variance and model selection through multi-head attention, but not for localization. However, TabPFN's bias can be improved further with a simple post-hoc localization method.
\end{itemize}
\cref{sec:discussion} concludes with suggestions for future research.
All proofs are given in \cref{sec:proofs}.

\subsection{Related Work}

\paragraph{In-context-learning in large language models}
The recent interest in ICL was spurred by the success of \emph{large language models (LLMs)}. These models are pre-trained on a sequence prediction task on a large corpus of text. When deployed, large models show the ability to solve tasks that they haven't seen during pre-training (e.g., mathematical reasoning problems), only from the prompt context \citep{NEURIPS2020_1457c0d6, wei2022finetuned}. In particular, no parameter updates are conducted after deployment. ICL has become a new paradigm for natural language processing and is intimately linked to the transformer architecture \citep{vaswani2017attention}.
\citet{dong2023survey} provide an up-to-date survey of the large body of LLM-related research on in-context learning.

\paragraph{In-context-learning on numeric data}
ICL has also been observed in more classical statistical learning tasks: classification and regression from tabular data. \citet{muller2021transformers} proposed the concept of PFNs and illustrated the abilities of a transformer model on toy examples. The TabPFN model of \citet{hollmann2022tabpfn} implements a matured version of this idea and shows superb performance on benchmarks with small tabular data. 
Concurrently, \citet{nguyen2022} proposed Transformer Neural Processes following essentially the same idea. They also consider non-\emph{iid} settings and show promising results in applications to image completion, contextual bandits, and Bayesian Optimization. This paper illuminates the statistical foundations of such models in the \emph{iid}-setting and disentangles the prior from the model architecture. 

\paragraph{Mechanics of transformer-based ICL}
\citet{garg2022can} show that transformers can learn target functions generated from linear models, two-layer neural networks, and decision trees. Several other works provide arguments and experiments on how in-context learning emerges through implicit gradient descent \citep{dai2022can,von2022transformers,akyurek2023what}. \citet{olsson2022context} identify a pattern of several attention heads working together, closely related to the discussion after \cref{thm:transformer-bias} in this paper. \citet{kirsch2022generalpurpose} experimentally investigates other architectural features (layers, memory, etc.). This work sheds further light on the mechanisms and architectural features enabling ICL.

Overall, the current work complements the existing ICL literature, by providing a theoretical foundation for the empirical findings and deriving new insights from the perspective of statistical generalization theory. 

\section{Theoretical Framework} \label{sec:framework}

\subsection{Statistical Model} \label{sec:DGP}

Consider a classification problem with class label $Y \in \Ycal$ and features $\bX \in \Xcal \subseteq \R^d$. Suppose we have \emph{iid} training data $\Dcal_n = (Y_{i}, \bX_{ i})_{i = 1}^{n}$ from some distribution $p_0$.
The goal is to predict the conditional class probabilities $p_0(y \mid \bx) = \Pr(Y = y \mid \bX = \bx)$.
From the perspective of Bayesian nonparametrics, we view $p_0$ as a realization of a random, infinite-dimensional parameter $p \in \Pcal$ with distribution $\Pi$. The distribution $\Pi$ is called \emph{prior} and expresses our beliefs about $p$ before seeing any data.
Under this model, data sets $\Dcal_n \cup (Y, \bX)$ are generated by the following mechanism:
\begin{enumerate}
	\item Draw $p \sim \Pi$.
	\item Draw \emph{iid} samples $\Dcal_n = (Y_i, \bX_i)_{i = 1}^n$ and $(Y, \bX) $ from model $p$.
\end{enumerate}

\subsection{Posterior Predictive Distribution}

For every $n$, the statistical model gives the tuple $(\Dcal_n \cup (Y, \bX), p)$ a well-defined joint distribution. For every $n$, we can then approximate $p_0(y \mid \bx)$ by the posterior predictive distribution (PPD)
\begin{align*}
	\pi(y \mid \bx, \Dcal_n) = \Pr(Y = y \mid \bX = \bx, \Dcal_n).
\end{align*}
This defines a family of PPDs indexed by $n$. 
If the prior $\Pi$ factorizes into independent parts for $p(y \mid \bx)$ and $p(\bx)$, the PPD can be written as
\begin{align} \label{eq:ppd}
	\pi(y \mid \bx, \Dcal_n) = \int p(y \mid \bx) d\Pi(p \mid \Dcal_n),
\end{align}
where the \emph{posterior} $\Pi(\cdot \mid \Dcal_n)$ is the conditional distribution of $p$ given the data  $\Dcal_n$.
The PPD $\pi(y \mid \bx, \Dcal_n)$ is then simply the posterior mean over conditional distributions $p(y \mid \bx)$. 

\begin{remark}
	In their implementation of PFNs, \citet{muller2021transformers} and \citet{hollmann2022tabpfn} use priors that factorize as above, but do not mention it explicitly to justify \eqref{eq:ppd}.
	Priors that do not factorize this way would lead to a different form of $\pi(y \mid \bx, \Dcal)$:
	\begin{align*}
		\pi(y \mid \bx, \Dcal_n) = \int p(y \mid \bx) d\Pi(p \mid \bx, \Dcal_n),
	\end{align*}
		Here, observing the test feature $\bx$ would be informative about the conditional distribution $p(y \mid \bx)$, which is unintuitive.
\end{remark}

\subsection{PFNs}

A PFN is a numerical approximation of the family of PPDs. It is
based on the insight that, for all $n$, the PPDs maximize the expected conditional likelihood
\begin{align} \label{eq:KL}
	\E_{\Pi} [\log q(Y \mid \bX, \Dcal_n)],
\end{align}
where $\E_{\Pi}$ is an expectation over $(Y, \bX) \cup \Dcal_n$ generated as in \cref{sec:DGP} \citep[see,][Section 3]{muller2021transformers}. 

\begin{theorem} \label{thm:idea}
	Let
	\begin{align*}
		\Qcal = \biggl\{q\colon (\Ycal \times \Xcal)_{i = 1}^{n + 1} \to [0, 1], \sum_{y \in \Ycal} q(y \mid \cdot, \cdot) = 1\biggr\}, 
	\end{align*} 
	denote the set of all conditional probability functions. Then  $\pi$ in \eqref{eq:ppd} satisfies
	$$\pi = \arg\max_{q \in \Qcal}\E_\Pi [\log q(Y \mid \bX, \Dcal_n)].$$
\end{theorem}
\begin{remark}
	Maximizing \eqref{eq:KL} can also be interpreted as minimizing expected KL divergence between $q(\cdot \mid \bX, \Dcal_n)$ and $\pi(\cdot \mid \bX, \Dcal_n)$.
\end{remark}
To approximate the PPDs, we train a model $q_{\btheta}$ parametrized by $\btheta$. To be precise, for every parameter value $\btheta$, there is an entire family of functions 
$$\{q_{\btheta, n}\colon (\Ycal \times \Xcal)_{i = 1}^{n + 1} \to [0, 1], n \in \N\},$$
but we shall not make this explicit in notation.
To find the best parameters for given PPDs, \citet{muller2021transformers} propose to solve
\begin{align} \label{eq:theta-true}
	\btheta^* = \arg\max_{\btheta} \E_{\Pi_N}\E_{\Pi} [\log q_{\btheta}(Y \mid \bX, \Dcal_N)],
\end{align}
where $\Pi_N$ is a probability distribution over the sample size $N$. 
The expectation over $N$ makes $q_{\btheta^*}$ mimic the \emph{family} of PPDs, not just its $n$th element.

The model $q_{\btheta}$ will normally be misspecified; that is, there is no parameter $\btheta$ such that $q_{\btheta}$ equals $\pi$.
In this case, \eqref{eq:theta-true} defines a KL-optimal approximation of $\pi$ over the class $\{q_{\btheta}\colon \btheta \in \Theta\}$.
In practice, the expectation in \eqref{eq:theta-true} is approximated by Monte-Carlo integration, i.e., averaging over \emph{iid} data sets $(Y_j, \bX_j) \cup \Dcal^{(j)}$ of size $N_j + 1$ generated as in \cref{sec:DGP} and $N_j \sim \Pi_N$. We approximate $\btheta^*$ by solving
\begin{align} \label{eq:erm}
	\wh \btheta = \arg\max_{\btheta} \sum_{j = 1}^m \log q_{\btheta}(Y_j \mid \bX_j, \Dcal^{(j)}).
\end{align}
This is of course an idealization of the training process. Sophisticated PFNs are large and usually trained in a single epoch. The maximum in \eqref{eq:erm} is never reached. This does not affect the main results of the following sections, which largely consider arbitrary $\btheta$.

\begin{remark}
	\citet{hollmann2022tabpfn} use a transformer network \citep{vaswani2017attention} for $q_{\btheta}$. For such architectures, any fixed network $q_{\btheta}$ accepts an arbitrary number of feature vectors $\bx_1, \dots, \bx_{n_\test}$ and a data set $\Dcal_{n}$ of arbitrary length. The output $q_{\btheta}( \cdot \mid \bx_1, \dots, \bx_{n_\test}, \Dcal_{n})$ are $n_{\test}$ vectors of conditional class probabilities. Each vector contains predictions for the conditional class probabilities $p(\cdot \mid \bx_j)$. The test size $n_{\test}$ is irrelevant in what follows, so we take $n_{\test} = 1$ for simplicity.
\end{remark}

\section{When PPDs can Learn} \label{sec:ppd}

The PPDs
\begin{align*}
	\pi(y\mid \bx, \Dcal_n)
	 & =  \int p(y\mid \bx) d\Pi(p \mid \Dcal_n)
\end{align*}
are fully characterized by the prior $\Pi$.
If $\Dcal_n$ is a data set generated from $p_0$, we hope that $\Pi(p \mid \bx, \Dcal_n)$ concentrates around $p_0$ as the size of $\Dcal_n$ increases.
Setting a good prior is tricky in a nonparametric context. Finding a prior supporting a large enough subset of possible functions isn't trivial. And even if, the prior may wash out very slowly or not at all if it puts too much mass in unfavorable regions \citep[see,][Sections 1.2--1.3]{ghosal2017fundamentals}.
But also if $p_0$ is outside the support $\Pcal = \{p\colon \Pi(p) > 0\}$ of $\Pi$, PPDs can learn from data if the prior is sufficiently well-behaved:
\begin{theorem} \label{thm:posterior-consistency}
	Under conditions \ref{a:unique_phi} and \ref{a:covering}, there is $p^* \in \Pcal$ such that
	\begin{align*} 
		 \pi(y\mid \bx, \Dcal_n)  \stackrel{n \to \infty}\rightarrow p^*(y \mid \bx) \quad \text{almost surely},
	\end{align*}
	for $P_0$-almost every $(y, \bx)$.
	Moreover, $p^*$ is a KL-optimal approximation of $p_0$ in $\Pcal$.
\end{theorem}
Exact conditions and a proof are given in \cref{apx:ppd}.
If $\Pcal$ is sufficiently large, the KL-optimal $p^* \in \Pcal$ is close to $p_0$.
This explains why PPDs can learn when fed more data. 
If this was not the case, trying to approximate them with PFNs would be pointless.
And the better we choose $\Pi$, the more attractive the PPDs become as an approximation target.

\section{PFN Approximation of the PPD}  \label{sec:approx}

Four factors influence the PFN approximation \eqref{eq:erm}: the data prior $\Pi$, the size prior $\Pi_N$, the model $q_{\btheta}$, and the Monte-Carlo size $m$.
Since a PFN is pre-trained, the model class $\{q_{\btheta}\colon \btheta \in \Theta\}$ can be considered fixed relative to the number of Monte-Carlo sets $m$. 
The approximation quality of $\wh \btheta$ then follows from standard results on empirical risk minimization. In particular, we can expect $\wh \btheta = \btheta^* + O_p(m^{-1/2})$, see \cref{sec:PFN-approx} for more details. 
The other factors are more interesting.

If $\Pi$ consists of only simple models, the optimal PFN $q_{\btheta^*}$
likely also produces only simple functions of $(y, \bx)$.
Conversely, simple models $\{q_{\btheta}\colon \btheta \in \Theta\}$ cannot benefit much from complex $\Pi$. For the pre-trained PFN to work well on diverse tasks, we need sufficient capacity in both $\{q_{\btheta}\colon \btheta \in \Theta\}$  and $\Pi$.

When pre-training the PFN via \eqref{eq:erm}, we sample data sets $\Dcal^{(j)}$ with random sample size $N_j$.
Let us define the KL-optimal parameter $\btheta^*_n$ for a given sample size:
\begin{align*}
	\btheta^*_{n} = \arg\max_{\btheta} \E_{\Pi} [\log q_{\btheta}(Y \mid \bX, \Dcal_{n})].
\end{align*}
The PPD $\pi(y \mid \bx, \Dcal_{n})$ we are trying to approximate changes with $n$. Hence, the KL-optimal parameter $\btheta^*_{n}$ may change with $n$ as well. 
Seen as a function of $(y, \bx)$, we should expect the complexity of $\pi(y \mid \bx, \Dcal_n)$ to increase in $n$. Similarly, we should expect the parameter $\btheta^*_{n}$ to favor more complex models. At the other extreme, $n = 1$, the true PPD is close to the average model in our prior and normally close to constant. 
The training set sizes $N_j$ can thus be seen as a regularizer on model complexity.
By optimizing an average over random sizes $N_j$, $\btheta_*$ also averages $\btheta^*_{N_j}$. The distribution $\Pi_N$ lets us emphasize some ranges of sample sizes more than others. 
The TabPFN of \citet{hollmann2022tabpfn} was trained with a uniform distribution over $\{1, \dots, 1023\}$ for $\Pi_N$. The restriction to small sample sizes has computational reasons: the cost of evaluating a transformer network scales quadratically in $N_j$. 

Since TabPFN has never seen sample sizes larger than around 1000 during pre-training, it is curious that it improves its predictions when fed larger data sets. Whether such behavior occurs depends in a non-obvious way on the family $\{q_{\wh \btheta}(\cdot \mid \cdot, \Dcal_n), n \in \N\}$.
The family learned by TabPFN seems to have some structure that allows extrapolating nicely to larger $n$. This structure may come from the architecture of the network $q_{\btheta}$ or from learning $\btheta^*$ for a given $\Pi$. The following section looks closer into the mechanisms at play.

\section{Why PFNs can Learn In-Context} \label{sec:learning}

There is no reason to believe the PFN behaves like a PPD family for some (implicit) prior when encountering sample sizes never seen in pre-training. So even though PPDs serve as a theoretical motivation for PFNs, \cref{thm:posterior-consistency} does not apply to $q_{\wh{\btheta}}$.
So why does a PFN $q_{\wh{\btheta}}$ pre-trained on up to $1000$ samples improve when feeding larger data sets during inference? 

\subsection{PFNs as Untrained Predictors} \label{sec:predictor}

To understand what is going on, we take a frequentist perspective. For any data size $n$ encountered at inference, we may treat the pre-tuned network $q_{\wh{\btheta}}(y \mid \bx, \cdot)$ as an untrained predictor for $p_0(y \mid \bx)$, i.e., a function $(\Ycal \times \Xcal)^n \to \Pcal_{Y \mid \bX}$ that maps a data set $\Dcal_n$ to an element of the space $\Pcal_{Y \mid \bX}$ of conditional distribution functions. In that view, $\btheta$ is a collection of tuning parameters of the predictor, selected through meta-learning in the pre-training phase.
Further, the `priors' $\Pi$ and $\Pi_N$ are simply distributions over tasks for which we want the predictor to do well.

Now decompose the estimation error into bias and variance components:
\begin{align*}
	 &\quad \quad \; q_{\btheta}(y \mid \bx, \Dcal_n)  - p_0(y \mid \bx) \\
     &=\quad  \underbrace{q_{\btheta}(y \mid \bx, \Dcal_n) - \E_{\Dcal_n \sim p_0^n}[q_{\btheta}(y \mid \bx, \Dcal_n)]}_{\text{variance}}  \\
     &\quad +\; \underbrace{\E_{\Dcal_n \sim p_0^n}[q_{\btheta}(y \mid \bx, \Dcal_n)]   - p_0(y \mid \bx)}_{\text{bias}}.
\end{align*}
Empirically, the error above decreases with $n$.
So what structural features of PFNs can explain this?

\subsection{Symmetry}

Standard transformers are symmetric functions of the individual samples in $\Dcal_n$. If the samples in $\Dcal_n$ are \emph{iid}, this is most natural.

\begin{lemma} \label{thm:symmetry}
	Let $f \colon (\Ycal \times \Xcal)^n  \to \Pcal_{Y \mid \bX}$ be any predictor. Then there is a symmetrized version $\wt f$ of $f$ such that, for every probability measure $P$,
	\begin{align*}
		\E_{\Dcal_n \sim P^n}[\wt f(\Dcal_n)] &= \E_{\Dcal_n \sim P^n}P[f(\Dcal_n)], \\ 
		\text{and} \quad \var_{\Dcal_n \sim P^n}[\wt f(\Dcal_n)] &\le \var_{\Dcal_n \sim P^n}[f(\Dcal_n)].
	\end{align*}
\end{lemma}
Thus, using symmetric $f$ is optimal in an MSE sense. However, symmetry itself does not have any meaningful consequences for learning.
For example, $q(y \mid \bx, \Dcal_n) = 1 / | \Ycal|$ is a symmetric function that is incapable of learning anything.

\subsection{Variance and Diminishing Sensitivity} \label{sec:variance}

There is other structure we can reasonably expect from $q_{\btheta}$.
When passed larger data sets $\Dcal_n$, the influence of individual elements should diminish. 	
This allows to bound the variance of the predictor $q_{\btheta}$. Formally,
suppose there are $\alpha > 0$ and $L < \infty$, such that for large enough $n$ and almost all data sets $\Dcal_n, \Dcal_n'$ differing only in one sample,
\begin{align} 
         \bigl| q_{\btheta}\bigl(y \mid \bx, \Dcal_n\bigr) - q_{\btheta}(y \mid \bx, \Dcal_n') \bigr|  
         \le L n^{-\alpha}. \label{eq:bdiff}
\end{align}
\begin{theorem} \label{thm:variance}
	If \eqref{eq:bdiff} holds, then
	\begin{align*}
		\bigl| q_{\btheta}(y \mid \bx, \Dcal_n) - \E[q_{\btheta}(y \mid \bx, \Dcal_n)] \bigr| \lesssim n^{1/2 - \alpha}
	\end{align*}
	with high probability.
\end{theorem}
If $\alpha > 1/2$, we get $\lim_{n \to \infty} n^{1/2 - \alpha} = 0$, so the variance caused by $\Dcal_n$ vanishes. 
In that case, the difference above vanishes almost surely.
\begin{lemma} \label{thm:variance-as}
	If \eqref{eq:bdiff} holds with $\alpha > 1/2$, then
	\begin{align*}
	  q_{\btheta}(y \mid \bx, \Dcal_n) - \E[q_{\btheta}(y \mid \bx, \Dcal_n)]  \stackrel{n \to \infty}\rightarrow 0 \quad \text{almost surely}.
	\end{align*}
\end{lemma}
This only partially explains how pre-trained PFNs can still learn at inference. The remaining error is due to bias. 

\subsection{Bias and the Need for Locality} \label{sec:bias}

The bias is determined by the behavior of the sequence $\E[q_{\btheta}(y \mid \bx, \Dcal_n)]$.
It is reasonable to assume that
\begin{align*}
	 \E[q_{\btheta}(y \mid \bx, \Dcal_n)] \stackrel{n \to \infty}\rightarrow  \bar q_{\btheta}(y \mid \bx),
\end{align*}
for some function $\bar q_{\btheta}$. Without a specific model $q_{\btheta}$ at hand, we cannot say much more. In \cref{sec:examples} we shall see examples where the bias is constant, and other examples where the bias decreases or increases with $n$.

We can give necessary conditions for a vanishing bias, however.
A predictor that has vanishing bias on a sufficiently rich class of functions must be \emph{local}: asymptotically, only samples $(Y_i, \bX_i) \in \Dcal_n$ with $\bX_i$ close to $\bx$ should contribute to $q_{\btheta}(y \mid \bx, \Dcal_n)$.
\begin{theorem} \label{thm:unbiased}
	Let $\Pcal$ be a set of distributions. Suppose that for every $p \in \Pcal$,
\begin{align} \label{eq:unbiased}
	\E_{\Dcal_n \sim p^n}[q_{\btheta}(y \mid \bx, \Dcal_n)] \stackrel{n \to \infty}\rightarrow p(y \mid \bx).
\end{align}
If \eqref{eq:bdiff} holds,  there is a sequence $\epsilon_n \to 0$ for every $\wt p \in \Pcal$, such that almost surely,
\begin{align} \label{eq:locality}
	 \bigl| q_{\btheta}(y \mid \bx, \Dcal_n) - 	 q_{\btheta}(y \mid \bx, \wt \Dcal_n)  \bigr| \stackrel{n \to \infty}\rightarrow  0,
\end{align}
where $\Dcal_n = (Y_i, \bX_i)_{i = 1}^n$ and $\wt \Dcal_n = (Y_i', \bX_i)_{i = 1}^n$ with 
\begin{align*}
	Y_i'\; \begin{cases}
		\; = Y_i, & \text{if } \|\bX_i - \bx\| \le \epsilon_n, \\
		\; \sim \wt p( \cdot \mid \bX_i), &  \text{if } \|\bX_i - \bx\| > \epsilon_n.
	\end{cases}
\end{align*}
\end{theorem}
So if $q_{\btheta}$ is unbiased for rich enough $\Pcal$, we can flip the labels of samples away from $\bx$ almost arbitrarily without changing the behavior of $q_{\btheta}$; only samples $(Y_i, \bX_i)$ with $\bX_i$ close to $\bx$ matter.

The result bears little meaning if the class $\Pcal$ is too small, and meaningless if $\Pcal$ contains only one $p$. Even for rich $\Pcal$, it only provides necessary conditions for a vanishing bias. A constant predictor $q_{\btheta} = 1$ is local in the sense of \eqref{eq:locality}, but its bias does not change with $n$.  However, if $\Pcal$ is rich and the bias does vanish for all $p \in \Pcal$, the predictor $q_{\btheta}$ can effectively only use $\epsilon_n n = o(n)$ samples, so \eqref{eq:bdiff} is unlikely to hold with $\alpha = 1$. This is in line with the lower bounds on the bias-variance trade-off derived in \citet{derumigny2020lower}. 

\section{Insights on Specific PFNs} \label{sec:examples}

We now consider some concrete examples of PFNs $q_{\btheta}$, to shed further light on the factors facilitating learning in the inference phase.
Before turning to transformer networks, we briefly discuss two simpler models to illustrate some key mechanisms. The following result will be helpful.

\begin{lemma} \label{lem:generic-smoother}
	Let $g$ be a function bounded by $K < \infty$ and
	\begin{align*}
		q_{\btheta}(y \mid \bx, \Dcal_n) = \frac{\sum_{i = 1}^n g(y, \bx, Y_i, \bX_i)}{\sum_{i = 1}^n  \ind\{\bX_i \in A_n(\bx)\}},
	\end{align*}
	for some sequence $A_n(\bx) \subset \Xcal$.
	If 
	$$n^\eta\Pr\{\bX_i \in A_n(\bx)\} \stackrel{n \to \infty}\rightarrow  c,$$ 
	for some $\eta \in (0, 1/2)$ and $c > 0$, then
	$q_{\btheta}$ satisfies the conditions of \cref{thm:variance} with $\alpha = 1 - \eta$ and $L = 4K / c$.
\end{lemma}

\subsection{Window Smoother} \label{sec:window}

For $\theta \in (0, \infty)$, define 
\begin{align*}
	q_{\theta}(y \mid \bx, \Dcal_n)
	= \frac{\sum_{i = 1}^n \ind(Y_i = y) \ind(\| \bX_i - \bx\| < \theta)}{\sum_{i = 1}^n  \ind(\| \bX_i - \bx\| < \theta)}.
\end{align*}
This corresponds to a window smoother with bandwidth $\theta$.
Then $\theta^*$ is the KL-optimal bandwidth for datasets from the prior. The fitted PFN is therefore just a window smoother with its hyperparameter tuned to such data sets. 
According to \cref{lem:generic-smoother}, $q_\theta$ satisfies \eqref{eq:bdiff} with $\alpha = 1$. The bias
\begin{align*}
	&\quad \: \E[q_{\theta}(y \mid \bx, \Dcal_n)]  - p_0(y \mid \bx) \\
    &= \Pr_0(Y = y \mid \| \bX - \bx\| < \theta ) - p_0(y \mid \bx)
\end{align*}
is constant, but optimized for data sizes from $\Pi \times \Pi_N$. 
Despite constant bias, the PFN learns from more data at inference, but only through reducing its variance. 

Now consider some sequence $(a_n)_{n \in \N}$ and
\begin{align*}
	q_{\theta}(y \mid \bx, \Dcal_n)
	= \frac{\sum_{i = 1}^n \ind(Y_i = y) \ind(\| \bX_i - \bx\| < a_n\theta)}{\sum_{i = 1}^n \ind(\| \bX_i - \bx\| < a_n\theta)}.
\end{align*}
If $a_n$ increases with $n$, the width of the smoothing window does too.
This choice of $ q_{\theta}$ isn't sensible, of course, as the bias
\begin{align*}
	\Pr_0(Y = y \mid \| \bX - \bx\| < a_n\theta ) - p_0(y \mid \bx)
\end{align*}
typically increases with $n$.
If we instead choose $a_n \to 0$ and $p_0$ is sufficiently smooth, we can get rid of the bias.
The scaling $a_n = n^{-1/(4 + d)}$ is known to be optimal in an MSE sense \citep[e.g.,][]{wand1994kernel}, with squared bias and variance decreasing at rate $n^{-4/(4 + d)}$. Indeed, we have $a_n\Pr_0(\| \bX - \bx\| < a_n\theta) \to p_0(\bx)$, so this is the convergence rate implied by \cref{lem:generic-smoother} and \cref{thm:variance}.
The hyperparameter $\theta^*$ reduces to a prefactor, tuned to be (asymptotically) optimal for data sets generated from $\Pi$ of arbitrary size.

\subsection{Classification Trees} \label{sec:bma}

To keep the notation simple, suppose for the moment that $\Xcal \subseteq \R$. 
Define
\begin{align*}
	&\quad \; q_{\btheta}(y \mid x, \Dcal) \\
	&= \sum_{j = 0}^{S}  \ind_{(\theta_j, \theta_{j + 1}]}(x)\frac{\sum_{i = 1}^n  \ind(Y_i = y) \ind_{(\theta_j, \theta_{j + 1}]}(X_i)}{\sum_{i = 1}^n   \ind_{(\theta_j, \theta_{j + 1}]}(X_i)},
\end{align*}
as a classification tree with parameters $\btheta \in \R^{S}$ and $\theta_0 = -\infty, \theta_{S + 1} = +\infty$ by convention. The (hyper-)parameters are the split locations of the tree. The split locations are trained offline, to work best on sets from the prior $\Pi \times \Pi_N$. \cref{lem:generic-smoother} yields that \cref{thm:variance} holds with $\alpha = 1$. The bias is
\begin{align*}
     \sum_{j = 1}^{S}  \ind_{(\theta_j, \theta_{j + 1}]}(x)  \Pr_0(Y = y \mid X \in ( \theta_{j},  \theta_{j + 1}]) - p_0(y \mid x),
\end{align*}
which is independent of $n$. To reduce the bias as in the previous example, we would need to grow the number $S$ of split locations with $n$. But the model is considered fixed in the inference phase and we cannot change the number of parameters.


Instead, we could set up an ensemble of classification trees with Bayesian model averaging. Let $q_{\btheta_1}, \dots, q_{\btheta_K}$ be classification trees as above, and
\begin{align*}
	q_{\btheta}(y \mid x, \Dcal_n) = \frac{1}{K} \sum_{k = 1}^K q_{\btheta_k}(y \mid x, \Dcal_n) w(\Dcal_n; \btheta_k),
\end{align*}
where
\begin{align*}
	w(\Dcal_n; \btheta_k) = \frac{\exp\{-\mathrm{BIC}(\Dcal_n; \btheta_k)\} }{\sum_{j = 1}^K \exp\{ -\mathrm{BIC}(\Dcal_n; \btheta_j)\}},
\end{align*}
and
\begin{align*}
	\mathrm{BIC}(\Dcal_n; \btheta_k) = -2\sum_{i = 1}^n \log q_{\btheta_k}(Y_i \mid X_i, \Dcal_n) + S  \log n.
\end{align*}
As $n$ grows, the model $q_{\btheta}$ drifts towards a weighted average of the best-performing ensemble members. We can thus expect the bias to reduce with $n$, approaching the bias of the best ensemble members. The role of the hyperparameters $\btheta$ is the same as for a single tree. But now, the KL-optimal parameter $\btheta^*$ likely induces more complex and diverse ensemble members.

\subsection{Transformer Networks}  \label{sec:transformers}

We now consider a transformer network with one layer. Let $\Dcal_n = \{\bV_i\}_{i = 1}^n$ with $\bV_i = (Y_i, \bX_i) \in \{0, 1\} \times \R^{d}$ and $\bv = (0, \bx)$. Similar\footnote{Some scaling and redundancies in the parametrization have been deliberately removed. They help for training the network, but not its theoretical analysis.} to \citet{thickstun2021transformer}, define 
\begin{align*}
	a_{j}^{(h)} & = \softmax \bigl(\bv^\top W_{q}^{(h)} \bV_1, \dots, \bv^\top W_{q}^{(h)}  \bV_{n}\bigr)_j, \\
	\bu' & = \sum_{h = 1}^H \sum_{j = 1}^{n} a_{j}^{(h)} W_{v}^{(h)} \bV_j,                                                             \\
	\bu & =  \layernorm (\bv + \bu'; \bgamma),                                                  \\
	\bz' & =  W_{r, 2} \relu (W_{r, 1} \bu; \bgamma),                                                    \\
	\bz  & =  \layernorm (\bu + \bz'; \bgamma),                                                  \\
	q_{\btheta}(\cdot \mid \bx, \Dcal_n) & = \softmax(W_o \bz),
\end{align*}
where  $W_{q}^{(h)}, W_{v}^{(h)} \in \R^{(d + 1) \times (d + 1)}$, $W_{r, 1}, W_{r, 2}^\top \in \R^{m \times (d + 1)}$, and $ W_{o} \in \R^{|\Ycal| \times (d + 1)}$.
The parameter $\btheta$ collects all these matrices.
The $\softmax$, $\layernorm$, and $\relu$ operations are defined as
\begin{align*}
	\softmax(\bv)            & = \frac{\exp(\bv)}{\sum_{j=1}^d \exp(v_j)} ,  \\
	\layernorm(\bv; \bgamma) & = \gamma_1 \frac{\bv - \avg(\bv)}{\|\bv - \avg(\bv)\| + |\gamma_2|} + \gamma_3 , \\
     \avg(\bv)    &= \frac 1 d \sum_{j = 1}^d v_j  \bm 1,  \\
	\relu(\bv)               & = \max(0, \bv),
\end{align*}
with $\exp$ and $\max$ acting componentwise on vectors.
Here and in everything that follows, the norm $\| \cdot \|$ is understood as $\| \cdot \|_2$ for both vectors and matrices.

The first two equations describe an \emph{attention mechanism} with $H$ heads \citep{vaswani2017attention}. By definition, $a_1^{(h)} + \cdots + a_n^{(h)} = 1$. The idea is that within every head, the attention weights $a_j^{(h)}$ emphasize specific samples $\bV_j \in \Dcal_n$. Emphasis is put on those $\bV_j$ that are `similar' to $\bv$ in a sense measured by $\bv^\top W_{q}^{(h)} \bV_j$. Each attention head allows for a different definition of similarity.
With the help of \cref{thm:variance} and \cref{thm:transformer-bdiff}, we can show that the variance of this predictor vanishes.
\begin{theorem} \label{thm:transformer-variance}
	For $\Xcal = \{\bx\colon \|\bx\| \le K\}$ and $\|W_{q}^{(h)}\|, \|W_{v}^{(h)}\|,  \|W_{r, 1}\|, \|W_{r, 2}\|, \|W_o\| < \infty$, it holds
	\begin{align*}
		\bigl| q_{\btheta}(y \mid \bx, \Dcal_n) - \E[q_{\btheta}(y \mid \bx, \Dcal_n)] \bigr| \lesssim n^{-1/2},
	\end{align*}
	with high probability.
\end{theorem}
The variance vanishes irrespective of the parameter $\btheta$. This is because the attention mechanism necessarily gives diminishing weight to individual samples (see the proof in \cref{sec:transformer-proof}).

The bias depends heavily on the choice of $\btheta$. 
\begin{theorem} \label{thm:transformer-bias}
	Under the assumptions of \cref{thm:transformer-variance},
	\begin{align*}
		 \E[q_{\btheta}(\cdot \mid \bx, \Dcal_n)] \stackrel{n \to \infty}\rightarrow \bar q_{\theta}(\cdot \mid \bx),
	\end{align*}
	where $\bar q_{\theta}(\cdot \mid \bx)$ is defined as $q_{\btheta}(\cdot \mid \bx, \Dcal_n)$, but with $\bu'$ replaced by 
	\begin{align*}
		\bar \bu' & = \sum_{h = 1}^H  W_{v}^{(h)} \E_{\bV \sim g_h}[ \bV],
		\\   \text{where} \quad g_h(\bs) &= \frac{\exp(\bv^\top W_{q}^{(h)} \bs)}{\E_{\bV \sim p_0}[\exp(\bv^\top W_{q}^{(h)} \bV)]} p_0(\bs).
	\end{align*}
\end{theorem}

\begin{figure*}
	\centering 
	\fig[0.9]{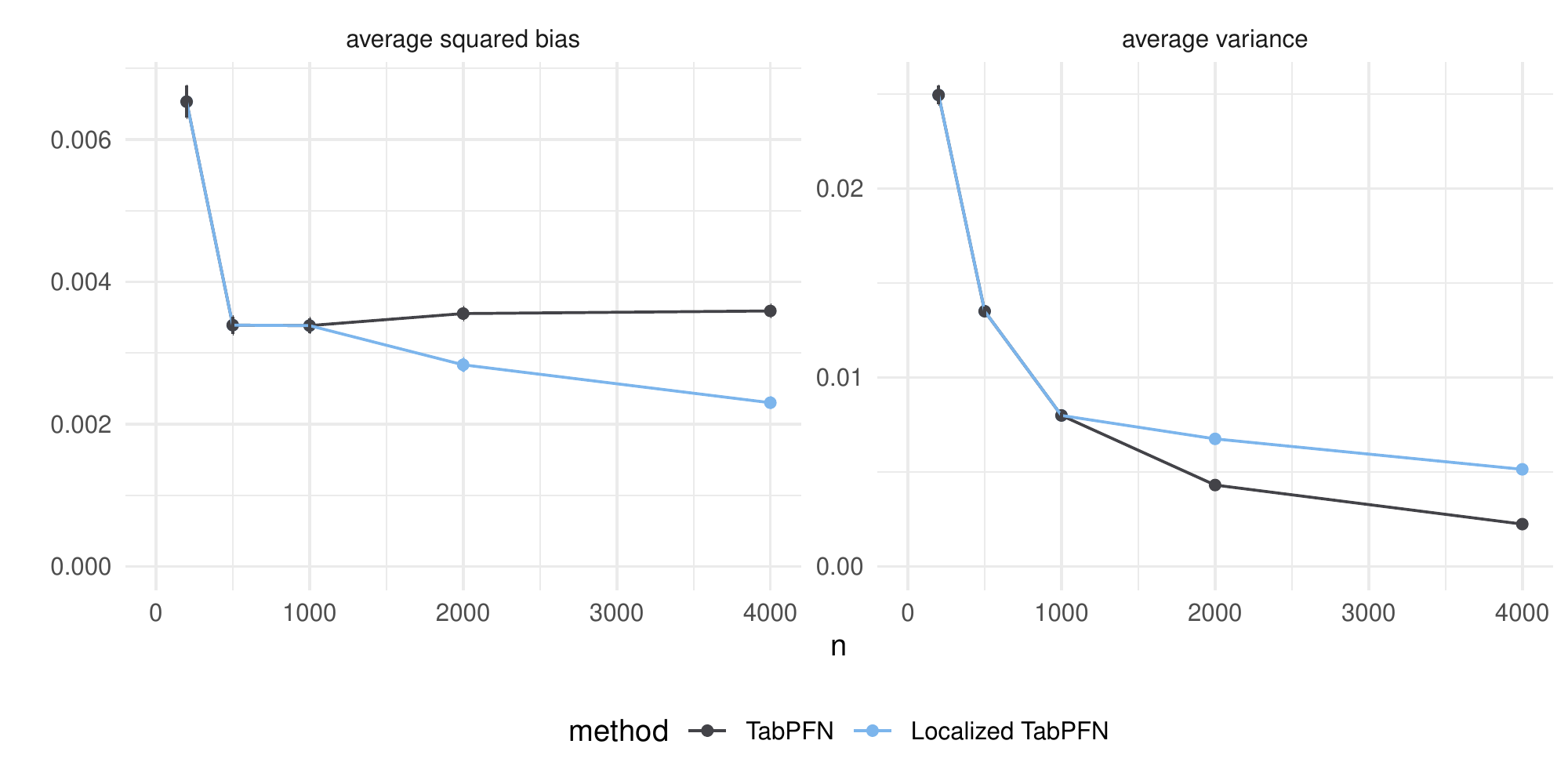}
	\caption{Average squared bias and variance of the pre-trained TabPFN of \citet{hollmann2022tabpfn} on simulated data sets.}
	\label{fig:bias-variance}
\end{figure*}

The measure $g_h$ is an exponentially tilted version of $p_0$ \citep{Siegmund76}.
The tilt can be understood as an infinitesimal form of the attention mechanism.
Relative to $p_0$, the tilted measure lifts the likelihood of values $\bs$ that are similar to $\bv$ (with similarity measured by $\bv^\top W_{q}^{(h)} \bs$) and discounts the likelihood of others. Each attention head $h$ assesses a certain aspect of the unknown distribution $p_0$ (characterized by $W_{q}^{(h)}$). If the matrices $(W_{q}^{(h)})_{h = 1}^H$ are specified well, the aspect views distinguish distinct feature values. This localizes the predictor to some degree, but not in the sense of \cref{thm:unbiased}. Although we upweight the influence of samples ``similar to'' $\bv$, there always remains an influence of samples away from $\bv$. We cannot flip the labels of such samples without changing the predictor (asymptotically), so we should not expect the bias to vanish.

Nevertheless, the limiting bias $\bar q_{\theta}(y \mid  \bx) - p_0(y \mid \bx)$ may be small if the remaining network processes \emph{the sum} of aspect summaries $W_{v}^{(h)} \E_{\bV \sim g_h}[\bV]$ into a good approximation of $p_0$. The relevance of individual aspect summaries depends on the true measure $p_0$, and less relevant aspects may also contribute less to the sum. On small samples, this effect is milder. At the extreme end, $n = 1$, all attention weights $a_j^{(h)}$ equal 1, so all aspects contribute equally. This suggests that the bias of the transformer network may decrease --- provided the hyperparameters downstream make meaningful use of the aspect views. For example, \citet{olsson2022context} identified powerful patterns of several attention heads working together. This effect is similar to that of the model averaging layer in \cref{sec:bma}. Key to this is the presence of multiple attention heads ($H > 1$). However, this applies only to sample sizes the parameter $\btheta$ has been tuned to. For larger sample sizes, there is no reason to expect a tuned network's bias to decrease.

\subsection{Localized PFNs} \label{sec:localizedPFN}

According to \cref{thm:unbiased}, we need to localize the network to make its bias decrease. A simple post-hoc approach applicable to any pre-trained network is the following. To predict the label at a new feature $\bx$:
\begin{enumerate}
	\item Construct a reduced training set $\Dcal_n(\bx)$ by excluding all but the $k_n$ nearest neighbors of $\bx$ from $\Dcal_n$.
	\item Predict the label that maximizes $q_{\btheta}(\cdot \mid \bx, \Dcal_n(\bx))$.
\end{enumerate}
Intuitively, restricting to a neighborhood is like stretching/flattening the target $p(y \mid \cdot)$ at the cost of a reduction in sample size. Flatter functions are easier to approximate. This is the mechanism behind the window smoother from \cref{sec:window}. If the model $q_\theta$ approximates constant functions well, localization should improve the bias. 



\subsection{Numerical Validation}

Since the key mechanism acting on $\Dcal_n$ remains intact if we add more layers to the network, the findings likely transfer to larger networks. The main predictions from our theoretical considerations are: $(i)$ the variance vanishes at rate $1/n$, $(ii)$ the bias does not vanish, but decreases until $n \approx 1000$. To confirm this empirically, we simulate 500 data sets $\Dcal_n$ from the model $p_0(1 \mid \bX) = 1/2 + \sin(\bm 1^\top \bX) / 2$ with $Y \in \{0, 1\}$, $\bX \sim \Ncal( \bnull, I_5)$, and run the pre-trained TabPFN of \citet[pip version 0.1.8]{hollmann2022tabpfn}.\footnote{An R script to reproduce the results can be found at 
\url{https://gist.github.com/tnagler/62f6ce1f996333c799c81f1aef147e72}.} We compute the average squared bias and variance over 100 samples $\bX_{\text{test}} \sim \Ncal(\bnull, I_5)$. The results in \Cref{fig:bias-variance} confirm that the variance indeed decreases at rate $1/n$ and that the bias decreases until $n \approx 1000$, but does not vanish. 

The analysis shows that, for larger sample sizes at inference, TabPFN learns mainly through decreasing its variance. This variance reduction is a consequence of the transformer architecture and takes place irrespective of the tuned parameters $\wh \btheta$. \Cref{fig:bias-variance} also shows the results of a localized version of TabPFN (as in \cref{sec:localizedPFN} with $k_n = \min\{500, \lceil n^{4/(d + 4)} \rceil\}$). Here, the bias continues to decrease beyond $n = 1000$ at the cost of a slightly larger variance.

\section{Discussion} \label{sec:discussion}

As explained in \cref{sec:predictor}, the prior $\Pi$ characterizes tasks we want the predictor to do well on. \citet{hollmann2022tabpfn} propose a new kind of prior based on structural causal models \citep{pearl2009causality} that is interesting on its own. Their intuitive idea is that the pair $(Y, \bX)$ is generated by some noisy, causal mechanism (not necessarily in the direction $\bX \to Y$). Because the mechanisms can be arbitrarily complex, the prior is essentially nonparametric. The rate at which corresponding posteriors contract is a complex issue \citep[Chapters 8--9]{ghosal2017fundamentals} and poses an interesting open question.

The key factors driving PFNs capability to learn are their sensitivity to individual samples, their ability to choose submodels, and localization. These insights may help to inform architecture design. In \cref{sec:window} and \cref{sec:localizedPFN}, we found a way to make the bias vanish at the cost of increased sensitivity to the training instances. This was achieved by introducing a scaling of the tuning parameters adapted to the training set size $n$. Whether this is possible and what a good scaling is depends on the model architecture. The localization approach in \cref{sec:localizedPFN} is simple and can be applied post-hoc to any pre-trained network $q_\theta$. More serious architecture design should account for the entire training pipeline and computational efficiency. Additional improvements can be expected if localization is incorporated into pre-training. 
Thinking about ways to adapt the transformer architecture appropriately could be a promising path. Another possible improvement is to augment the architecture with a Bayesian averaging mechanism similar to \cref{sec:bma}.

\citet{hollmann2022tabpfn} acknowledge constraints on the feature dimension and sample size as a major limitation of current PFN implementations. Owing to the standard transformer architecture, the maximal feature size is fixed, and the algorithm scales quadratically in the number of samples. To mitigate this, several works proposed scalable modifications of the transformer architecture \citep{beltagy2020longformer,zaheer2020big,kitaev2020reformer}.
\citet{hollmann2022tabpfn} rightfully point out that PFNs are quick enough to be used as ensemble members. The size constraints could therefore be overcome by boosting and bagging techniques akin to random forests or boosted trees.  

The full potential of PFNs is yet to be explored.

\section*{Acknowledgements}

The author is grateful for several helpful comments by Samuel M\"uller, Noah Hollmann, Thibault Vatter, and two anonymous referees.

\bibliography{ref}
\bibliographystyle{icml2023}

\newpage
\appendix
\onecolumn
\appendix

\section{Proofs} \label{sec:proofs}

\subsection{Proof of \texorpdfstring{\Cref{thm:idea}}{Theorem \ref{thm:idea}}}

By definition of the KL-divergence,
\begin{align*}
	0 \le \mathrm{KL}\bigl[q(\cdot \mid \bx, \Dcal)  \;||\; \pi(\cdot \mid \bx, \Dcal) \bigr] \quad \text{for all } (\bx, \Dcal),
\end{align*}
or, equivalently,
\begin{align*}
	\E_{Y \sim \pi(\cdot \mid \bx, \Dcal)}[\log \pi(Y \mid \bx, \Dcal)] \ge \E_{Y \sim \pi(\cdot \mid \bx, \Dcal)}[\log q(Y \mid \bx, \Dcal)]  \quad \text{for all } (\bx, \Dcal).
\end{align*}
Since this holds for any $(\bx, \Dcal)$ it must also hold if we take expectations over random draws of $(\bx, \Dcal)$. Taking expectation with respect to $(\bx, \Dcal) \sim \Pi$ on both sides, the law of iterated expectations yields
\begin{align*}
	\E_\Pi [\log \pi(Y \mid \bX, \Dcal)] \ge \E_\Pi [\log q(Y \mid \bX, \Dcal)]. \tag*{\qedsymbol}
\end{align*}

\subsection{Proof of \texorpdfstring{\Cref{thm:posterior-consistency}}{Theorem \ref{thm:posterior-consistency}}}
\label{apx:ppd}

Write $\KL(f \mid f')$ 
for the KL-divergence of $f$ relative to $f'$, and $H(f, f')$ 
for their Hellinger distance. We need the following assumptions:
\begin{enumerate}[label=(A\arabic*)]
	\item \label{a:unique_phi}
	      There is a unique $p^* \in \Pcal$ with $p^* = \arg\min_{p} \KL(p^* \mid p_0)$ and $\KL(p^* \mid p_0) < \infty$.
	\item \label{a:covering}
	      For every $\alpha \in (0, 1/2)$, there are sets $B_1, \dots, B_{J(\alpha)}$ with
	      \begin{align*}
		      \Pcal \subseteq \bigcup_{j = 1}^{J(\alpha)}  B_j, \qquad \sup_{p, p' \in B_{j}}  H(p, p') \le 4(\alpha^2/2)^{1/\alpha}, \qquad \sum_{j = 1}^{J(\epsilon)} \Pi(B_j)^\alpha < \infty.
	      \end{align*}
\end{enumerate}

Now let $\Pi_n(A) = \int_A d\Pi(p \mid \Dcal_n)$ be the posterior measure.
From our assumptions and  Corollary 1 of \citet{de2013bayesian} it follows that for all $\epsilon > 0$,
\begin{align} \label{eq:deblasi}
	\Pi_n\bigl\{p \in \Pcal\colon H(p, p^*) > \epsilon\bigr\} \to 0,
\end{align}
with probability 1 over sequences $\Dcal_n$.
For some $\delta > 0$ and an arbitrary set $A \subseteq \Ycal \times \Xcal$ with $\mu_A = P^*(A) > 0$, define
\begin{align*}
	S_{\delta, A} = \biggl\{p \in \Pcal \colon \inf_{(y, \bx) \in A} \bigl|p(y \mid \bx) - p^*(y \mid \bx) \bigr| \ge \delta\biggr\}.
\end{align*}
For any $p \in S_{\delta, A}$, it holds
\begin{align*}
	\mu_A \delta = \int_A \delta p^*(\bx) d\bx
	 & \le  \int_A	|p^*(y \mid \bx) - p(y \mid \bx)| p^*(\bx) d\bx                     \\
	 & =  \int_A	\bigl|p^*(y, \bx) - p(y \mid \bx)p^*(\bx)\bigr| d\bx                    \\
	 & =  \int_A	\bigl|p^*(y, \bx) - p(y, \bx) + p(y \mid \bx)[p(\bx) - p^*(\bx)]\bigr| d\bx \\
	 & \le  \int	\bigl|p^*(y, \bx) - p(y, \bx)| d\bx + \int \bigl|p(\bx) - p^*(\bx)\bigr| d\bx \\
	 & \le 4 \operatorname{TV}(p^*, p)                                            \\
	 & \le 8 H(p^*, p),
\end{align*}
where $\operatorname{TV}(f, f')$ 
is the total variation distance.
Together with \eqref{eq:deblasi}, this implies
\begin{align*}
	\Pi_n(S_{\delta, A}) \le \Pi_n\bigl\{ p \colon H(p, p^*) \ge \mu_A \delta / 8 \bigr\} \to 0.
\end{align*}
We then get
\begin{align*}
	\inf_{(y, \bx) \in A} |\pi(y \mid \bx, \Dcal_n) - p_0(y \mid \bx)|
	 & \le  \inf_{(y, \bx) \in A} \biggl|\int p(y \mid \bx) d\Pi_n(p) - p_0(y \mid \bx)\biggr|                           \\
	 & \le  \inf_{(y, \bx) \in A} \int  \bigl|p(y \mid \bx) - p_0(y \mid \bx) \bigr|   d\Pi_n(p)                         \\
	 & \le  \inf_{(y, \bx) \in A} \int_{S_{\delta, A}} \bigl|p(y \mid \bx) - p_0(y \mid \bx) \bigr|   d\Pi_n(p) + \delta \\
	 & \le 2\Pi_n(S_{\delta, A}) + \delta \to \delta \quad \text{almost surely}.
\end{align*}
Since $\delta$ and $A$ were arbitrary, we have shown that
\begin{align*}
	\pi(y\mid \bx, \Dcal_n)  \to p*(y \mid \bx),
\end{align*}
with probability 1 for $P^*$-almost every $(y, \bx)$. Since $\KL(p^* \mid p_0) < \infty$  by \ref{a:unique_phi}, convergence must also take place for $P_0$-almost every $(y, \bx)$. \hfill \qedsymbol

\subsection{Proof of \texorpdfstring{\Cref{thm:symmetry}}{Theorem \ref{thm:symmetry}}}

Let $R_n$ be the set of permutations $\rho\colon (\Ycal \times \Xcal)^n \to  (\Ycal \times \Xcal)^n$. Define the symmetrized function $\wt f$ as
$$\wt f(\Dcal_n) = \frac 1 {|R_n|} \sum_{\rho \in R_n} (f \circ \rho)(\Dcal_n).$$
If the elements in $\Dcal_n$ are \emph{iid}, it holds  $\E_P[f(\Dcal_n)] = \E_P[(f \circ \rho)(\Dcal_n)]$ and  $\var_P[f(\Dcal_n)] = \var_P[(f \circ \rho)(\Dcal_n)] $ for any $\rho \in R_n$.
Therefore, $\E_P[f(\Dcal_n)] = \E_P[\wt f(\Dcal_n)]$ and
\begin{align*}
	\var_P[\wt f(\Dcal_n)]
	 & = \frac{\var_P[f(\Dcal_n)]}{|R_n|^2} \sum_{\rho, \rho' \in R_n} \corr_P[(f \circ \rho)(\Dcal_n), (f \circ \rho')(\Dcal_n)]  		\le \var_P[f(\Dcal_n)],
\end{align*}
with equality for all $P$ if and only if $f$ is symmetric. \hfill \qedsymbol

\subsection{Proof of \texorpdfstring{\Cref{thm:variance}}{Theorem \ref{thm:variance}}}
\label{sec:variance-proof}

McDiarmid's inequality \citep{mcdiarmid_1989} yields
\begin{align*}
\Pr\bigl(	\bigl| q_{\btheta}(y \mid \bx, \Dcal_n) - \E[q_{\btheta}(y \mid \bx, \Dcal_n)] \bigr| > \epsilon\bigr) \le 2\exp\biggl(-\frac{2 \epsilon^2}{L^2 n^{1 - 2\alpha}}\biggr).
\end{align*}
Choosing $\epsilon^2 = \log(\delta)L^2 n^{1 - 2\alpha} / 2$, we get
\begin{align*}
	\Pr\bigl(	\bigl| q_{\btheta}(y \mid \bx, \Dcal_n) - \E[q_{\btheta}(y \mid \bx, \Dcal_n)] \bigr| > \log(\delta)^{1/2}K L n^{1/2 - \alpha} / \sqrt{2} \bigr) \le 2\delta. \tag*{\qedsymbol}
\end{align*}

\subsection{Proof of \texorpdfstring{\Cref{thm:variance-as}}{Lemma \ref{thm:variance-as}}}
\label{sec:variance-as-proof}
Using the first inequality from the previous proof, we get
\begin{align*}
	\sum_{n = 1}^\infty \Pr\bigl(	\bigl| q_{\btheta}(y \mid \bx, \Dcal_n) - \E[q_{\btheta}(y \mid \bx, \Dcal_n)] \bigr| > \epsilon\bigr) \le 2 	\sum_{n = 1}^\infty \exp\biggl(-\frac{2 \epsilon^2}{L^2 n^{1 - 2\alpha}}\biggr) < \infty.
\end{align*}
The Borel-Cantelli lemma then implies that, almost surely,
\begin{align*}
	\lim_{n \to \infty} \bigl| q_{\btheta}(y \mid \bx, \Dcal_n) - \E[q_{\btheta}(y \mid \bx, \Dcal_n)] \bigr| \le  \epsilon
\end{align*}
Since $\epsilon$ was arbitrary the claim follows. \hfill \qedsymbol

\subsection{Proof of \texorpdfstring{\Cref{thm:unbiased}}{Theorem \ref{thm:unbiased}}}
Condition \eqref{eq:unbiased} implies 
\begin{align*}
	 \E_p[q_{\btheta}(y \mid \bx, \Dcal_n)] - 	 \E_{\wt p}[q_{\btheta}(y \mid \bx, \Dcal_n)]   \to 0,
\end{align*}
for all $p, \wt p \in \Pcal$ with $p(y \mid \bs) = \wt p(y \mid \bs)$ for $\| \bs - \bx\| < \epsilon$. \cref{thm:variance-as} then implies 
\begin{align*}
	\lim_{n \to \infty}\bigl| q_{\btheta}(y \mid \bx, \Dcal_n) - q_{\btheta}(y \mid \bx, \wt \Dcal_n) \bigr|  = 0 \quad \text{almost surely}.
\end{align*}
Since $\epsilon$ was arbitrary, convergence must also hold for some sequence $\epsilon_n \to 0$. \hfill \qedsymbol

\subsection{Proof of \texorpdfstring{\Cref{lem:generic-smoother}}{Lemma \ref{lem:generic-smoother}}}

Let $\Dcal_n = (Y_i, \bX_i)_{i = 1}^n$ and $\Dcal_n' = (Y_i', \bX_i')_{i = 1}^n$ such that 
$(Y_i, \bX_i) = (Y_i', \bX_i')$ for all $i > 1$.
Hoeffdings's inequality \citep[][Theorem 2.8]{boucheron2013concentration} gives
\begin{align*}
	\Pr\biggl(\biggl| \frac 1 n \sum_{i = 1}^n  \ind\{\bX_i \in A_n(\bx)\} -\Pr\{\bX_i \in A_n(\bx)\} \biggr| >  c n^{-\eta} / 3 \biggr) &\le 2\exp\bigl(-c^2 n^{1-2\eta} / 9\bigr).
\end{align*}
Since $\eta < 1/2$, 
\begin{align*}
	\sum_{n = 1}^\infty \exp\bigl(-c^2 n^{1-2\eta} / 9\bigr) < \infty,
\end{align*}
and the Borell-Cantelli lemma implies that for large $n$,
\begin{align*}
	\biggl| \frac 1 n \sum_{i = 1}^n  \ind\{\bX_i \in A_n(\bx)\} -\Pr\{\bX_i \in A_n(\bx)\} \biggr| \le  c n^{-\eta} / 3  \quad \text{almost surely}.
\end{align*}
The remaining inequalities are understood almost surely, for large enough $n$.
Because $n^\eta\Pr\{\bX_i \in A_n(\bx)\} \to c$, we get
\begin{align*}
	&\quad n^\eta \biggl| \frac 1 n \sum_{i = 1}^n  \ind\{\bX_i \in A_n(\bx)\} - c n^{-\eta} \biggr| \\ 
	&\le n^\eta \biggl| \frac 1 n \sum_{i = 1}^n  \ind\{\bX_i \in A_n(\bx)\} - \Pr\{\bX_i \in A_n(\bx)\} \biggr| + \biggl|	n^\eta\Pr\{\bX_i \in A_n(\bx)\} - c \biggr| \\
	&\le c / 2 ,
\end{align*}
which implies
\begin{align*}
\frac 1 n \sum_{i = 1}^n  \ind\{\bX_i \in A_n(\bx)\} \ge cn^{-\eta}/2.
\end{align*}
Using this bound, 
\begin{align*}
\quad n^{1- \eta}|q_{\theta}(y \mid \bx, \Dcal_n) - q_{\theta}(y \mid \bx, \Dcal_n')| 
\le 2 |g(y, \bx, Y_1, \bX_1) - g(y, \bx, Y_1', \bX_1')| / c \le 4K / c,
\end{align*}
which proves the claim. \hfill \qedhere

\subsection{Proof of \texorpdfstring{\Cref{thm:transformer-variance}}{Theorem \ref{thm:transformer-variance}}}
\label{sec:transformer-proof}

The theorem is a consequence of \cref{thm:variance} and the following result.   (The norm bounds on $\|\bx\|$ and the weight matrices are arbitrary and can be relaxed.)
\begin{theorem} \label{thm:transformer-bdiff}
	Let $\Xcal = \{\bx\colon \|\bx\| \le 1\}$ and $\|W_{q}^{(h)}\|, \|W_{v}^{(h)}\|,  \|W_{r, 1}\|, \|W_{r, 2}\|, \|W_o\| \le 1$.
	Then the network $q_{\btheta}$ satisfies \eqref{eq:bdiff} with $\alpha = 1$ and $L = O(H|\gamma_1| / |\gamma_2|)$.
\end{theorem}
\begin{proof}
	Let $\wt \Dcal_n = (\Dcal_n \setminus \bV_n) \cup \wt \bV_n$ and define $\wt a_j^{(h)}$, $\wt \bu'$, $\wt \bu$, $\wt \bz'$, $\wt \bz$ accordingly.
	Because $\softmax$ is 1-Lipschitz \citep[Proposition 4]{gao2017properties},
	\begin{align*}
		|q_{\btheta}(y \mid \bx, \Dcal_n) - q_{\btheta}(y \mid \bx, \wt \Dcal_n)| & \le \| W_o (\bz - \wt \bz)\|
		\le \| \bz - \wt \bz\|.
	\end{align*}
	Using \cref{lem:layernorm} below, we further get
	\begin{align*}
		\bigl\| \bz - \wt \bz\bigr\| & \le \bigl\| \layernorm(\bu + \bz') - \layernorm(\wt \bu' + \wt \bz)\bigr\|  \le 4 \frac{|\gamma_1|}{|\gamma_2|}  \biggl(\bigl\|\bu - \wt \bu\bigr\|  + \bigl\|\bz'- \wt \bz'\bigr\| \biggr).
	\end{align*}
	Because also $\relu$ is 1-Lipschitz,
	\begin{align*}
		\bigl\|\bz'- \wt \bz'\bigr\| & = \bigl\|W_{r, 2}\relu(W_{r, 1}\bu)- W_{r, 2}\relu(W_{r, 1}\wt \bu)\bigr\|  \le \|\bu - \wt \bu\|.
	\end{align*}
	Using \cref{lem:layernorm} again,
	\begin{align*}
		\| \bu - \wt \bu\|
		 & = \| \layernorm(\bv + \bu') - \layernorm(\bv + \wt \bu')\| \le 4\frac{|\gamma_1|}{|\gamma_2|} \| \bu' - \wt  \bu'\|.  
	\end{align*}
	The last displays together yield
	\begin{align} \label{eq:tv1}
		|q_{\btheta}(y \mid \bx, \Dcal_n) - q_{\btheta}(y \mid \bx, \wt \Dcal_n)| \le 32 \frac{|\gamma_1|}{|\gamma_2|} \bigl\|\bu - \wt \bu\bigr\|.
	\end{align}
	Defining $\wt \bV_i = \bV_i$ for $i < n$, we obtain 
	\begin{align*}
		 \frac 1 H \|\bu' - \wt \bu'\| 
		& = \frac 1 H  \biggl\| \sum_{h = 1}^H \biggl[ \sum_{j = 1}^{n} a_j^{(h)}W_{v}^{(h)} \bV_j - \sum_{j = 1}^{n} \wt a_j^{(h)}W_{v}^{(h)} \wt \bV_j \biggr] \biggr\|                               \\
		& \le   \max_{1 \le h \le H} \biggl\|  \sum_{j = 1}^{n} a_j^{(h)}W_{v}^{(h)} \bV_j - \sum_{j = 1}^{n} \wt a_j^{(h)}W_{v}^{(h)} \wt \bV_j \biggr\| \\
		& = \max_{1 \le h \le H} \biggl\|   \sum_{j = 1}^{n} a_j^{(h)}W_{v}^{(h)} (\bV_j - \wt \bV_j) + \sum_{j = 1}^{n} (a_j^{(h)} - \wt a_{j}) W_{v}^{(h)} \wt \bV_j \biggr\|      \\
		& \le  \max_{1 \le h \le H}\sum_{j = 1}^{n} a_j^{(h)}\|W_{v}^{(h)}\| \|\bV_j - \wt \bV_j\| + \sum_{j = 1}^{n} |\wt a_j^{(h)}-  a_{j}| \| W_{v}^{(h)}\| \| \wt \bV_j\| \\
		& \le 4 \max_{1 \le h \le H}  a_n^{(h)}   + 2 \sum_{j = 1}^{n} |\wt a_j^{(h)}-  a_{j}|.
	\end{align*}
	Let $s_{i} = \bv^\top W_{q}^{(h)} \bV_i$ and note that $|s_i| \le \|\bv\| \|W_{q}^{(h)}\| \max_j\|\bV_j\| \le 4$.
	\begin{align*}
		|a_n^{(h)}| = \frac{\exp(s_n) }{\sum_{j = 1}^n \exp(s_j)} \le \frac{e^4}{\sum_{j = 1}^n e^{-4}} = \frac{e^8}{n}.
	\end{align*}
	Further,
	\begin{align*}
		|\wt a_j^{(h)} - a_j^{(h)}| & =  \biggl| \frac{\exp(s_j) }{\sum_{j = 1}^n \exp(s_j)} - \frac{\exp(\wt s_j) }{\sum_{j = 1}^{n} \exp(\wt s_j)} \biggr|                                                                                                                           \\
		                & =  \biggl| \frac{\exp(s_j) }{\sum_{j = 1}^n \exp(s_j)} - \frac{\exp(\wt s_j) }{\sum_{j = 1}^n \exp(s_j)} + \frac{\exp(\wt s_j) }{\sum_{j = 1}^n \exp(s_j)} - \frac{\exp(\wt s_j) }{\sum_{j = 1}^{n} \exp(\wt s_j)} \biggr|                       \\
		                & \le  \biggl| \frac{\exp(s_j) - \exp(\wt s_j) }{\sum_{j = 1}^n \exp(s_j)}  \biggr| + \biggl|  \frac{\exp(\wt s_j) }{\sum_{j = 1}^n \exp(s_j)} - \frac{\exp(\wt s_j) }{\sum_{j = 1}^{n} \exp(\wt s_j)} \biggr|                                     \\
		                & =  \biggl| \frac{\exp(s_j) - \exp(\wt s_j) }{\sum_{j = 1}^n \exp(s_j)}  \biggr| +  \frac{\exp(\wt s_j)}{\sum_{j = 1}^n \exp(s_j)} \biggl| \frac{\sum_{j = 1}^n \exp(s_j) - \sum_{j = 1}^n \exp(\wt s_j)}{\sum_{j = 1}^{n} \exp(\wt s_j)} \biggr| \\
		                & =  \biggl| \frac{\exp(s_j) - \exp(\wt s_j) }{\sum_{j = 1}^n \exp(s_j)}  \biggr| +  \frac{\exp(\wt s_j)}{\sum_{j = 1}^n \exp(s_j)} \biggl| \frac{\exp(s_n) -  \exp(\wt s_n)}{\sum_{j = 1}^{n} \exp(\wt s_j)} \biggr|.
	\end{align*}
	The first term on the right is zero if $j \neq n$. For $j = n$, it can be bounded by $2e^8 / n$ as before. Using the same argument for the second term above, we get
	\begin{align*}
		|\wt a_j^{(h)} - a_j^{(h)}| \le \frac{2e^8}{n} \ind(j = n) + \frac{2e^{16}}{n^2}.
	\end{align*}
	Accordingly,
	\begin{align} \label{eq:tv3}
		\frac 1 H \|\bu' - \wt \bu'\|
		\le  4 \max_{1 \le h \le H}  a_n^{(h)}   + 2 \sum_{j = 1}^{n} |\wt a_j^{(h)}-  a_{j}|
		\le  \frac{4e^8}{n} + \frac{2e^8}{n} + n \frac{2e^{16}}{n^2} \le \frac{8e^{16}}{n}. 
	\end{align}
	Combining \eqref{eq:tv1} and \eqref{eq:tv3} proves the claim.
\end{proof}

\begin{lemma} \label{lem:layernorm}
	For any two vectors $\ba, \bb \in \R^d$ it holds
	\begin{align*}
		\bigl\| \layernorm(\ba; \bgamma) - \layernorm(\bb; \bgamma) \bigr\| & \le 4\frac{|\gamma_1|}{|\gamma_2|}\|\ba - \bb\|.
	\end{align*}
\end{lemma}
\begin{proof}
	Let us first assume $\avg(\ba) = \avg(\bb) = \bnull$ and, without loss of generality, $\|\ba\| \ge \|\bb\|$. It holds,
	\begin{align*}
	      & \quad \; \bigl\| \layernorm(\ba; \bgamma) - \layernorm(\bb; \bgamma) \bigr\| \\
			& = |\gamma_1| \biggl\|\frac{\ba }{\|\ba\|+|\gamma_2|} - \frac{\bb }{\|\bb\|+|\gamma_2|} \biggr\|                                                 \\  
			 & =  |\gamma_1|\biggl\|\frac{\ba}{\|\ba\|+|\gamma_2|} - \frac{\bb}{\|\ba\|+|\gamma_2|} + \frac{\bb}{\|\ba\| +|\gamma_2|}  - \frac{\bb}{\|\bb\|+|\gamma_2|} \biggr\|      \\
			 & \le  |\gamma_1| \frac{\|\ba - \bb\|}{\|\ba\|+|\gamma_2|}  +  |\gamma_1|\|\bb\| \biggl|\frac{1}{\|\ba\|+|\gamma_2|}  - \frac{1}{\|\bb\|+|\gamma_2|} \biggr| \\
			 & =  |\gamma_1|\frac{\|\ba - \bb\|}{\|\ba\|+|\gamma_2|}  +  |\gamma_1| \frac{\|\bb\|}{\|\bb\| + |\gamma_2|}\frac{\bigl|\|\bb\| - \|\ba\| \bigr|}{\|\ba\|+|\gamma_2|}                   \\
			& \le  |\gamma_1|\frac{\|\ba - \bb\|}{\|\ba\|+|\gamma_2|}  +  |\gamma_1|\frac{\bigl|\|\bb\| - \|\ba\| \bigr|}{\|\ba\|+|\gamma_2|}   \\
			& \le   2\frac{|\gamma_1|\|\ba - \bb\|}{\|\ba\| +|\gamma_2|} \qquad 	[\text{reverse triangle inequality}] \\
			& \le 2\frac{|\gamma_1|}{|\gamma_2|} \|\ba - \bb\|.
	\end{align*}
	If $\avg(\ba) \neq 0$ or $\avg(\bb) \neq 0$, we get
	\begin{align*}
		\| \layernorm(\ba; \bgamma) - \layernorm(\bb; \bgamma) \| & \le 2\frac{|\gamma_1|}{|\gamma_2|}\|\ba - \bb - \avg(\ba - \bb)\| \le 4\frac{|\gamma_1|}{|\gamma_2|}\|\ba - \bb \|,
	\end{align*}
	because
	\begin{align*}
		\|\avg(\ba - \bb) \|^2 = \biggl\|\frac 1 d \sum_{i = 1}^d (a_i - b_i) \bm 1 \biggr\|^2   \le \biggl(\frac 1 d \sum_{i = 1}^d |a_i - b_i| \|\bm 1 \| \biggr)^2  = d \biggl(\frac 1 d \sum_{i = 1}^d |a_i - b_i| \biggr)^2   
		  \le \|\ba - \bb\|^2,
	\end{align*}
	where we used the triangle inequality in the second step and Jensen's inequality in the last.
\end{proof}

\subsection{Proof of \texorpdfstring{\Cref{thm:transformer-bias}}{Theorem \ref{thm:transformer-bias}}}

We have  
\begin{align*}
\bu' = \sum_{h = 1}^H \sum_{j = 1}^{n} a_{j}^{(h)} W_{v}^{(h)} \bV_j  
&= \sum_{h = 1}^H \frac{\frac 1 n \sum_{j = 1}^{n} \exp(\bv^\top W_{q}^{(h)} \bV_j) W_{v}^{(h)} \bV_j}{\frac 1 n \sum_{j = 1}^{n} \exp(\bv^\top W_{q}^{(h)} \bV_j) }. 
\end{align*}
The law of large numbers implies $\bu' \to \bar \bu'$ almost surely, where 
\begin{align*}
	\bar \bu' = 	\sum_{h = 1}^H \frac{\E_{\bV \sim p_0}[\exp(\bv^\top W_{q}^{(h)} \bV)  W_{v}^{(h)} \bV]}{\E_{\bV \sim p_0}[\exp(\bv^\top W_{q}^{(h)} \bV)]}.
\end{align*}
Now, observe that
\begin{align*}
	\frac{\E_{\bV \sim p_0}[\exp(\bv^\top W_{q}^{(h)} \bV)  W_{v}^{(h)} \bV]}{\E_{\bV \sim p_0}[\exp(\bv^\top W_{q}^{(h)} \bV)]} 
	&=  \int   W_{v}^{(h)} \bs  \underbrace{\frac{ \exp(\bv^\top W_{q}^{(h)} \bs)}{\E_{\bV \sim p_0}[\exp(\bv^\top W_{q}^{(h)} \bV)]}  p_0(\bs)}_{g_h(\bs)} d\bs  =  W_{v}^{(h)} \E_{\bV \sim g_h}[ \bV].
\end{align*} 
Since all following operations on $\bu'$ are continuous (see the proof of \cref{thm:transformer-bdiff}), it also holds 
\begin{align*}
	\lim_{n \to \infty} \bigl| q_{\btheta}(\cdot \mid \bx, \Dcal_n) - \bar q_{\btheta}(\cdot \mid \bx) \bigr| = 0 \quad \text{almost surely}. 
\end{align*}
Because the sequence $|q_{\btheta}(\cdot \mid \bx, \Dcal_n) - q_{\btheta}(\cdot \mid \bx)|$ is uniformly bounded by 2, convergence is also in $L_1$. This implies $	\lim_{n \to \infty} \E[q_{\btheta}(\cdot \mid \bx, \Dcal_n)] = \bar q_{\theta}(\cdot \mid \bx).$ \hfill \qedsymbol



\section{Approximation Results for the PFN Parameter} \label{sec:PFN-approx}

Suppose that the parameters $\btheta$ live in a subset of $\R^p$ with $p$ fixed and finite. This assumption would be questionable in the usual machine learning setting, but our situation is different. A PFN is pre-trained offline, using $m$ Monte-Carlo samples from data sets $\Dcal^{(j)}$. Given enough computing power, we can take $m$ as large as we want. 

Given sufficient regularity, the following theorems are direct applications of existing results. To keep additional concepts and notation to a minimum, we omit detailed conditions and proofs and refer to the original works for specifics. We start with the behavior of $\wh \btheta$.
\begin{theorem}  \label{thm:asym-theta}
	It holds:
	\begin{enumerate}[label=(\roman*)]
		\item $	\wh{\btheta} \to \btheta^* $ in probability,
		\item $\sqrt{m}(\wh{\btheta} - \btheta^*) \to \Ncal(\bnull, \Sigma_{\btheta^*})$ in distribution, where $\Sigma_{\btheta^*} = I_{\btheta^{*}}^{-1} V_{\btheta^*}  I_{\btheta^{*}}^{-1}$ and
		      \begin{align*}
			                        & I_{\btheta^*} =   \E_{\Pi_N}\E_{\Pi}[\nabla_{\btheta}^2 \log q_{\btheta^*}(Y \mid \bX, \Dcal_N)]                                                                 \\
			      \mbox{and}  \quad & V_{\btheta^*} =  \E_{\Pi_N}\E_{\Pi}[\nabla_{\btheta} \log q_{\btheta^*}(Y \mid \bX, \Dcal) \times \nabla_{\btheta}^\top \log q_{\btheta^*}(Y \mid \bX, \Dcal_N)].
		      \end{align*}
	\end{enumerate}
\end{theorem}
\begin{proof}
	See Corollary 3.2.3 and Example 3.2.12 in \citet{vaart1996weak}.
\end{proof}
The first part shows that $\wh{\btheta}$ is a valid approximation of $\btheta^*$, the second quantifies its accuracy.
\cref{thm:asym-theta} has direct implications for the approximated model $q_{\wh{\btheta}}$.

\begin{theorem}  \label{thm:asym-q}
	It holds:
	\begin{enumerate}[label=(\roman*)]
		\item
		      $\displaystyle \sup_{(y, \bx, \Dcal)}|q_{\wh{\btheta}}(y\mid \bx, \Dcal) - q_{\btheta^*}(y \mid \bx, \Dcal)| \to 0$ in probability,
		\item $\sqrt{m}(q_{\wh{\btheta}} - q_{\btheta^*})$ converges weakly to a mean-zero Gaussian process $\G$ with
		      \begin{align*}
			      \cov\bigl(\G(y\mid \bx, \Dcal), \G(y'\mid \bx', \Dcal')  \bigr)
			      = \nabla_{\btheta} q_{\btheta^*}(y\mid \bx, \Dcal)^\top \Sigma_{\btheta^*} \nabla_{\btheta} q_{\btheta^*}(y'\mid \bx', \Dcal').
		      \end{align*}
	\end{enumerate}
\end{theorem}
\begin{proof}
	Part $(i)$ follows from \cref{thm:asym-theta} and the continuous mapping theorem \citep[Theorem 1.11.1 ]{vaart1996weak}, part $(ii)$ from the delta method \citep[Theorem 3.9.4]{vaart1996weak}.
\end{proof}

From the second part, we see that the variance of $q_{\wh{\btheta}}(y\mid \bx, \Dcal)$ is approximately
\begin{align*}
	\frac 1 m \nabla_{\btheta} q_{\btheta^*}(y\mid \bx, \Dcal)^\top \Sigma_{\btheta^*} \nabla_{\btheta} q_{\btheta^*}(y\mid \bx, \Dcal).
\end{align*}
Intuitively, the variance depends on the accuracy of $\wh{\btheta}$ (through $\Sigma_{\btheta^*}$), and the sensitivity of $q_{\btheta^*}$ with respect to $\btheta^*$ (through $\nabla_{\btheta} q_{\btheta^*}$). Model complexity normally works against us in both parts. Hence, more complex models need to be trained with more Monte-Carlo samples to limit the variance.


\end{document}